\documentclass{miri-tech-article}

\bibliography{MIRIBibliography.bib,MIRIBibliography-temp.bib}

\usepackage{amsmath,amssymb,amsthm,amsopn,amsfonts}
\usepackage[colorinlistoftodos]{todonotes}
\usepackage[ruled,vlined,noend]{algorithm2e}
\usepackage{multicol} \usepackage{fancyhdr} \usepackage{pdfpages}
\usepackage{enumerate}

\newtheorem{theorem}{Theorem}[section]
\newtheorem{lemma}[theorem]{Lemma}

 \newenvironment{definition}[1][Definition]{\begin{trivlist}
 \item[\hskip \labelsep {\bfseries #1}]}{\end{trivlist}}

\DeclareMathOperator*{\Exp}{\mathbb{E}}
\DeclareMathOperator*{\Prob}{\mathbb{P}}
\DeclareMathOperator*{\argmax}{argmax}

\title{Reflective Oracles: A Foundation for Classical Game Theory}

\author{
Benja Fallenstein \and Jessica Taylor \\
Machine Intelligence Research Institute \\
\{benja,jessica\}@intelligence.org\\
\And
Paul F.\ Christiano\\
UC Berkeley\\
paulfchristiano@eecs.berkeley.edu
}

\begin{document}

\publishingnote{This is an extended version of \cite{Fallenstein:2015a}.}

\maketitle

\begin{abstract}
Classical game theory treats players as special---a description of a game contains a full, explicit enumeration of all players---even though in the real world, ``players'' are no more fundamentally special than rocks or clouds. It isn't trivial to find a decision-theoretic foundation for game theory in which an agent's coplayers are a non-distinguished part of the agent's environment. Attempts to model both players and the environment as Turing machines, for example, fail for standard diagonalization reasons.
    
In this paper, we introduce a ``reflective'' type of oracle, which is able to answer questions about the outputs of oracle machines with access to the same oracle. These oracles avoid diagonalization by answering some queries randomly. We show that machines with access to a reflective oracle can be used to define rational agents using causal decision theory. These agents model their environment as a probabilistic oracle machine, which may contain other agents as a non-distinguished part.
    
We show that if such agents interact, they will play a Nash equilibrium, with the randomization in mixed strategies coming from the randomization in the oracle's answers. This can be seen as providing a foundation for classical game theory in which players aren't special.
\end{abstract}

  \section{Introduction}

    Classical decision theory and game theory are founded on the notion of a perfect Bayesian reasoner \cite{Myerson:1997}. Such an agent may be uncertain which of several possible worlds describes the state of its environment, but given any particular possible world, it is able to deduce exactly what outcome each of its available actions will produce \cite{Savage:1972}. This assumption is, of course, unrealistic \cite{Gigerenzer:2001,Gaifman:2004}: Agents in the real world must necessarily be boundedly rational reasoners, which make decisions with finite computational resources. Nevertheless, the notion of a perfect Bayesian reasoner provides an analytically tractable first approximation to the behavior of real-world agents, and underlies an enormous body of work in statistics~\cite{Box:2011}, economics~\cite{Kreps:1990}, computer science~\cite{Korb:2003}, and other fields.
    
    On closer examination, however, the assumption that agents can compute what outcome each of their actions leads to in every possible world is troublesome even if we assume that agents have unbounded computing power. For example, consider the game of \emph{Matching Pennies}, in which two players each choose between two actions (``heads'' and ``tails''); if the players choose the same action, the first player wins a dollar, if they choose differently, the second player wins. Suppose further that both players' decision-making processes are Turing machines with unlimited computing power. Finally, suppose that both players know the exact state of the universe at the time they begin deliberating about the actions they are going to choose, including the source code of their opponent's decision-making algorithm.\footnote{The technique of quining (Kleene's second recursion theorem \cite{Rogers:1967}) shows that it is possible to write two programs that have access to each other's source code.}
    
    In this set-up, by assumption, both agents know exactly which possible world they are in. Suppose that they are able to use this information to accurately predict their opponent's behavior. Since both players' decision-making processes are deterministic Turing machines, their behavior is deterministic given the initial state of the world; each player either definitely plays ``heads'' or definitely plays ``tails''. But neither of these possibilities is consistent: For example, if the first player chooses heads and the second player can predict this, the second player will choose tails, but if the first player can predict this in turn, it will choose tails, contradicting the assumption that it chooses heads.
    
    The problem is caused by the assumption that given its opponent's source code, a player can figure out what action the opponent will choose. One might think that it could simply run its opponent's source code, but if the opponent does the same, both programs will go into an infinite loop. Binmore~\cite{Binmore:1987}, discussing the philosophical justification for game-theoretic concepts such as Nash equilibrium, puts this problem as follows:
    \begin{quote}
    In any case, if Turing machines are used to model the players, it is possible to suppose that the play of a game is prefixed by an exchange of the players' G\"odel numbers\dots Within this framework, a perfectly rational machine ought presumably to be able to predict the behavior of the opposing machines perfectly, since it will be familiar with every detail of their design.  And a universal Turing machine \emph{can} do this.  What it \emph{cannot} do is predict its opponents' behavior perfectly \emph{and} simultaneously participate in the action of the game.  It is in this sense that the claim that perfect rationality is an unattainable ideal is to be understood.
    \end{quote}
    Even giving the players access to a halting oracle does not help, because even though a machine with access to a halting oracle can predict the behavior of an ordinary Turing machine, it cannot in general predict the behavior of another oracle machine.
    
    Classical game theory resolves this problem by allowing players to choose \emph{mixed} strategies (probability distributions over actions); for example, the unique Nash equilibrium of Matching Pennies is for each player to assign ``heads'' and ``tails'' probability~$0.5$ each. However, instead of treating players' decision-making algorithms as computable processes which are an ordinary part of a world with computable laws of physics, classical game theory treats players as special objects. For example, to describe a problem in game-theoretic terms, we must provide an explicit list of all relevant players, even though in the real world, ``players'' are ordinary physical objects, not fundamentally distinct from objects such as rocks or clouds.
    
    In this paper, we show that it is possible to define a certain kind of \emph{probabilistic} oracle---that is, an oracle which answers some queries non-deterministically---such that a Turing machine with access to this oracle can perform perfect Bayesian reasoning about environments that can themselves be described as oracle machines with access to the same oracle. This makes it possible for players to treat opponents simply as an ordinary part of this environment.
    
    When an environment contains multiple agents playing a game against each other, the probabilistic behavior of the oracle may cause the players' behavior to be probabilistic as well. We show that in this case, the players will always play a Nash equilibrium, and for every particular Nash equilibrium there is an oracle that causes the players to behave according to this equilibrium.  In this sense, our work can be seen as providing a foundation for classical game theory, demonstrating that the special treatment of players in the classical theory is not fundamental.

    The oracles we consider are not halting oracles; instead, roughly speaking, they allow oracle machines with access to such an oracle to determine the probability distribution of outputs of other machines with access to the same oracle. Because of their ability to deal with self-reference, we refer to these oracles as \emph{reflective oracles}.

  \section{Reflective Oracles} \label{section:reflective-oracles}

  In many situations, programs would like to predict the output of
  other programs.  They could simulate the other program in order to
  do this.  However, this method fails when there are cycles
  (e.g. program A is concerned with the output of program B which is
  concerned with the output of program A).  Furthermore, if a
  procedure to determine the output of another program existed, then
  it would be possible to construct a liar's paradox of the form ``if
  I return 1, then return 0, otherwise return 1''.

  These paradoxes can be resolved by using probabilities.  Let $\mathcal{M}$ be the set of \emph{probabilistic oracle machines}, defined here as Turing machines which can execute special instructions to (i)~flip a coin that has an arbitrary rational probability of coming up heads, and to (ii)~call an oracle $O$, whose behavior might itself be probabilistic.
  
  Roughly speaking, the oracle answers questions of the form: ``Is the probability that machine $M$ returns $1$ greater than $p$?'' Thus, $O$ takes two inputs, a machine $M\in\mathcal{M}$ and a rational probability $p\in[0,1]\cap\mathbb{Q}$, and returns either~$0$ or~$1$. If $M$ is guaranteed to halt and to output either~$0$ or~$1$ itself, we want $O(M,p) = 1$ to mean that the probability that $M$ returns $1$ (when run with $O$) is at least $p$, and $O(M,p) = 0$ to mean that it is at most $p$; if it is equal to $p$, both conditions are true, and the oracle may answer randomly. In summary,
  $$\begin{aligned}\Prob(M^O() = 1) \;>\; p \;\;\Longrightarrow\;\; \Prob(O(M, p) = 1) \;=\; 1\\\Prob(M^O() = 1) \;<\; p \;\;\Longrightarrow\;\; \Prob(O(M, p) = 0) \;=\; 1\end{aligned}$$
  where we write $\Prob(M^O() = 1)$ for the probability that $M$ returns $1$ when run with oracle $O$, and $\Prob(O(M,p) = 1)$ for the probability that the oracle returns $1$ on input $(M,p)$. We assume that different calls to the oracle are stochastically independent events (even if they are about the same pair $(M,p)$); hence, the behavior of an oracle $O$ is fully specified by the probabilities $\Prob(O(M,p) = 1)$.

  \begin{definition}
  A \emph{query} (with respect to a particular oracle~$O$) is a pair $(M,p)$, where $p\in[0,1]\cap\mathbb{Q}$ and $M^O()$ is a probabilistic oracle machine which almost surely halts and returns an element of~$\{0,1\}$.
  \end{definition}
  
  \begin{definition}
  An oracle is called \emph{reflective on $R$}, where $R$ is a set of queries, if it satisfies the two conditions displayed above for every $(M,p)\in R$. It is called \emph{reflective} if it is reflective on the set of all queries.
  \end{definition}
  
  \begin{theorem} \label{theorem:existence}
  
    (i) There is a reflective oracle.
    
    (ii) For any oracle~$O$ and every set of queries~$R$, there is an oracle~$O'$ which is reflective on~$R$ and satisfies $\Prob(O'(M,p) = 1) = \Prob(O(M,p) = 1)$ for all $(M,p)\notin R$.
  \end{theorem}
  \begin{proof}
  For the proof of (ii), see Appendix~\ref{appendix:existence-proof}; see also Theorem~\ref{theorem:existence-from-nash-equilibria}, which gives a more elementary proof of a special case. Part (i) follows from part (ii) by choosing~$R$ to be the set of all queries and letting~$O$ be arbitrary.
  \end{proof}
  
  As an example, consider the machine given by $M^O() = 1 - O(M,0.5)$, which implements a version of the liar paradox by asking the oracle what it will return and then returning the opposite.  By the existence theorem, there is an oracle which is reflective on $R = \{(M,0.5)\}$. This is no contradiction: We can set $\Prob(O(M, 0.5) = 1) = \Prob(O(M, 0.5) = 0) = 0.5$, leading the program to output~1 half the time and~0 the other half of the time.

  \section{From Reflective Oracles to Causal Decision Theory}
  
  We now show how reflective oracles can be used to implement a perfect Bayesian reasoner. We assume that each possible environment that this agent might find itself in can likewise be modeled as an oracle machine; that is, we assume that the laws of physics are computable by a probabilistic Turing machine with access to the same reflective oracle as the agent. For example, we might imagine our agent as being embedded in a Turing-complete probabilistic cellular automaton, whose laws are specified in terms of the oracle.
  
  We assume that each of the agent's hypotheses about which environments it finds itself in can be modeled by a (possibly probabilistic) ``world program''~$H^O()$, which simulates this environment and returns a description of what happened. We can then define a machine $W^O()$ which samples a hypothesis~$H$ according to the agent's probability distribution and runs $H^O()$. In the sequel, we will talk about~$W^O()$ as if it refers to a particular environment, but this machine is assumed to incorporate subjective uncertainty about the laws of physics and the initial state of the world.
  
  We further assume that the agent's decision-making process, $A^O()$, can be modeled as a probabilistic oracle machine embedded in this environment. As a simple example, consider the world program
  $$W^O() = \begin{cases}
    \$20 & \text{if } A^O() = 0 \\
    \$15 & \text{otherwise}
  \end{cases}
  $$
  In this world, the outcome is \$20 (which in this case means the agent
  receives \$20) if the agent chooses action 0 and \$15 if
  the agent chooses action 1.
  
  Our task is to find an appropriate implementation of $A^O()$. Here, we consider agents implementing causal decision theory (CDT) \cite{Weirich:2012}, which evaluates actions according to the consequences they cause: For example, if the agent is a robot embedded in a cellular automaton, it might evaluate the expected utility of taking action~$0$ or~$1$ by simulating what would happen in the environment if the output signal of its decision-making component were replaced by either~$0$ or~$1$.
  
  We will assume that the agent's model of the counterfactual consequences of taking different actions~$a$ is described by a machine $W_A^O(a)$, satisfying $W^O() = W_A^O(A^O())$ since in the real world, the agent takes action $a = A^O()$. In our example,
  $$W_A^O(a) = \begin{cases}
    \$20 & \text{if } a = 0 \\
    \$15 & \text{otherwise}
  \end{cases}
  $$
  We assume that the agent has a utility function over outcomes,~$u(\cdot)$, implemented as a lookup table, which takes rational values in~$[0,1]$.\footnote{Since the meaning of utility functions is invariant under affine transformations, the choice of the particular interval $[0,1]$ is no restriction.} Furthermore, we assume that both $W_A^O(0)$ and $W_A^O(1)$ halt almost surely and return a value in the domain of~$u(\cdot)$. Causal decision theory then prescribes choosing the action that maximizes expected utility; in other words, we want to find an~$A^O()$ such that
  $$A^O() \,=\, \argmax_{a}\, \mathbb{E}\left[u\left(W^O_A(a)\right)\right]$$
  In the case of ties, any action maximizing utility is allowed, and it is acceptable for $A^O()$ to randomize.
  
  We cannot compute this expectation by simply running~$u(W_A^O(a))$ many times to obtain samples, since the environment might contain other agents of the same type, potentially leading to infinite loops. However, we can find an optimal action by making use of a reflective oracle. This is easiest when the agent has only two actions ($0$ and $1$), but similar analysis extends to any number of actions. Define a machine
  $$E^O() :=
  \text{flip}\left(\frac{u(W_A^O(1)) - u(W_A^O(0)) + 1}{2}\right)$$ where
  $\text{flip}(p)$ is a probabilistic function that returns 1 with probability
  $p$ and 0 with probability $1 - p$.  
  \begin{theorem} \label{theorem:reflective-iff-utility-maximizing}
    $O$ is reflective on $\{(E,1/2)\}$ if and only if $A^O() := O(E, 1/2)$ returns a utility-maximizing action.
  \end{theorem}
  \begin{proof} The demand that $A^O()$ return a utility-maxmizing action is equivalent to
  $$\begin{aligned}\Exp[u(W_A^O(1))] > \Exp[u(W_A^O(0))] \;\Longrightarrow\; A^O() = 1\\
    \Exp[u(W_A^O(1))] < \Exp[u(W_A^O(0))] \;\Longrightarrow\; A^O() = 0\end{aligned}$$
  We have
    $$\Prob(E^O() = 1) = \Exp\left[\frac{u(W_A^O(1)) - u(W_A^O(0)) + 1}{2}\right]$$
    It is not difficult to check that $\Exp[u(W_A^O(1))] \gtrless \Exp[u(W_a^O(0))]$ iff $\Prob(E^O() = 1) \gtrless 1/2$. Together with the definition of $A^O()$, we can use this to rewrite the above conditions as
  $$\begin{aligned}\Prob(E^O() = 1) > 1/2 \;\Longrightarrow\; O(E,1/2) = 1\\
    \Prob(E^O() = 1) < 1/2 \;\Longrightarrow\; O(E,1/2) = 0\end{aligned}$$
    But this is precisely the definition of ``$O$ is reflective on $\{(E,1/2)\}$''.
  \end{proof}
  
  In order to handle agents which can choose between more than two actions, we can compare action 0 to action 1, then compare action 2 to the best of actions 0 and 1, then compare action 3 to the best of the first three actions, and so on. Adding more actions in this fashion does not substantially change the analysis.

  \section{From Causal Decision Theory to Nash Equilibria}
  
  Since we have taken care to define our agents' world models $W_A^O(a)$ in such a way that they can embed other agents,\footnote{More precisely, we have only required that $W_A^O(a)$ always halt and produce a value in the domain of the utility function $u(\cdot)$. Since all our agents do is to perform a single oracle call, they always halt, making them safe to call from $W_A^O(a)$.} we need not do anything special to pass from single-agent to multi-agent settings.  
  As in the single-agent case, we model the environment by a program $W^O()$  that contains embedded agent programs $A_1^O, \dotsc, A_n^O$ and returns an outcome. We can make the dependency on the agent program explicit by writing $W^O() = F^O(A_1^O(),\dotsc,A_n^O())$ for some oracle machine $F^O(\cdots)$. This allows us to define machines $W_i^O(a_i) := F^O(a_i,A_{-i}^O()) := F(A_1^O(),\dotsc,A_{i-1}^O(),a_i,A_{i+1}^O(),\dotsc,A_n^O())$, representing the causal effects of player~$i$ taking action~$a_i$.
  
  We assume that each agent has a utility function~$u_i(\cdot)$ of the same type as in the previous subsection. Hence, we can define the agent programs~$A_i^O()$ just as before:
  \begin{align*}
    A_i^O() &= O(E_i, 1/2)
    \\
    E_i^O() &= \text{flip}\left(\frac{u_i(W_i^O(1)) - u_i(W_i^O(0)) + 1}{2} \right)
  \end{align*}
  Here, each $E_i^O()$ calls $W_i^O()$, which calls $A_j^O()$ for each $j\neq i$, which refers to the source code of $E_j^O()$, but again, Kleene's second recursion theorem shows that this kind of self-reference poses no theoretical problem~\cite{Rogers:1967}.
  
  This setup very much resembles the setting of normal-form games. In fact:
  \begin{theorem} \label{theorem:nash-equilibrium-condition}
    Given an oracle $O$, consider the $n$-player normal-form game in which the payoff of player~$i$, given the pure strategy profile $(a_1,\dotsc,a_n)$, is $\Exp[u_i(F^O(a_1,\dotsc,a_n))]$. The mixed strategy profile given by $s_i := \Prob(A_i^O() = 1)$ is a Nash equilibrium of this game if and only if $O$ is reflective on $\{(E_1,1/2),\dotsc,(E_n,1/2)\}$.
  \end{theorem}

  \begin{proof}
  For $(s_1,\dotsc,s_n)$ to be a Nash equilibrium is equivalent to every player's mixed strategy being a best response; i.e., a pure strategy $a_i$ can only be assigned positive probability if it maximizes
  $$\Exp[u_i(F^O(a_i,A_{-i}^O()))] \;=\; \Exp[u_i(W_i^O(a_i))]$$
  By an application of Theorem~\ref{theorem:reflective-iff-utility-maximizing}, this is equivalent to $O$~being reflective on $\{(E_i,1/2)\}$.
  \end{proof}
  
  Note that, in particular, any normal-form game with rational-valued payoffs can be represented in this way by simply choosing $F^O$ to be the identity function. In this case, the theorem shows that every reflective oracle (which exists by Theorem~\ref{theorem:existence}) gives rise to a Nash equilibrium. In the other direction, Theorem~\ref{theorem:nash-equilibrium-condition} together with Theorem~\ref{theorem:existence}(ii) show that for any Nash equilibrium $(s_1,\dotsc,s_n)$ of the normal-form game, there is a reflective oracle such that $\Prob(A_i^O() = 1) = s_i$.

  \section{From Nash Equilibria to Reflective Oracles}

    In the previous section, we showed that a reflective oracle can be used to find Nash equilibria in arbitrary normal-form games.  It is interesting to note that we can also go in the other direction: For finite sets $R$ satisfying certain conditions, we can construct normal-form games $G_R$ such that the existence of oracles reflective on $R$ follows from the existence of Nash equilibria in $G_R$. This existence theorem is a special case of Theorem~\ref{theorem:existence}, but it not only provides a more elementary proof, but also provides a constructive way of finding such oracles (by applying any algorithm for finding Nash equilibria to $G_R$).
  
  \begin{definition}
    A set $R$ of queries is \emph{closed} if for every $(M,p)\in R$ and every oracle $O$, $M^O()$ is guaranteed to only invoke the oracle on pairs $(N,q)\in R$. It is \emph{bounded} if there is some bound $B_R\in\mathbb{N}$ such that for every $(M,p)\in R$ and every oracle $O$, $M^O()$ is guaranteed to invoke the oracle at most~$B_R$ times.
  \end{definition}
  
  \begin{definition}
    Given a finite set $R = \{(M_1,p_1),\dotsc,(M_n,p_n)\}$ and a vector $\vec x\in[0,1]^n$, define $O_{\vec x}$ to be the oracle satisfying $\Prob(O_{\vec x}(M_i,p_i) = 1) = x_i$ for $i = 1,\dotsc,n$, and $\Prob(O_{\vec x}(M,p) = 1) = 0$ for $(M,p)\notin R$.
  \end{definition}
  
  \begin{theorem} \label{theorem:existence-from-nash-equilibria}
    For any finite, closed, bounded set~$R = \{(M_1,p_1),\dotsc,(M_n,p_n)\}$, there is a normal form game $G_R$ with $m := n\cdot(2B_R + 1)$ players, each of which has two pure strategies, such that for any Nash equilibrium strategy profile $(s_1,\dotsc,s_m)$, the oracle $O_{\vec x}$ with $\vec x := (s_1,\dotsc,s_n)$ is reflective on~$R$.
  \end{theorem}

  \begin{proof}
    We divide the $n\cdot(2B_R + 1)$ players in our game into three sets: the \emph{main players} $i = 1,\dotsc,n$, the \emph{copy players} $g(i,j) := j\cdot n + i$, and the \emph{auxiliary players} $h(i,j) := (B_R + j)\cdot n + i$, for $i = 1,\dots,n$, $j = 1,\dotsc,B_R$.
    
    The mixed strategy~$s_i$ of a main player~$i$ will determine the probability that $O_{\vec x}(M_i,p_i) = 1$. We will force $s_{g(i,j)} = s_i$, i.e., we will force the mixed strategy of each copy player to equal that of the corresponding main player; thus, the copy players will provide us with independent samples from the $\mathrm{Bernoulli(s_i)}$ distribution, allowing us to simulate up to $B_R$ independent calls to $O(M_i,p_i)$. Finally, the auxiliary players are used to enforce the constraint $s_{g(i,j)} = s_i$, by having the copy player $g(i,j)$ play a variant of Matching Pennies against the auxiliary player $h(i,j)$.
    
    In order to define the game's payoff function, note first that by writing out each possible way that the at most $B_R$ oracle calls of $M_i^{O_{\vec x}}()$ might come out, we can write the probability that this machine returns~$1$ as a polynomial,
    $$\Prob(M_i^{O_{\vec x}}() = 1) \;=\; \sum_{k=1}^{K} c_{i,k} \prod_{i'=1}^n x_{i'}^{d_{i,k,i'}}$$
    where $d_{i,k,i'}\le B_R$. We want to force the main player~$i$ to choose pure strategy~$1$ if this probability is strictly greater than~$p_i$, pure strategy~$0$ if it is strictly smaller.
    
    To do so, we set player~$i$'s payoff function~$u_i(\vec a)$ to
    $$u_i(\vec a) \;=\; \begin{cases}\sum_{k=1}^K f_{i,k}(\vec a), & \text{if } a_i = 1, \\ p_i, & \text{otherwise}\end{cases}$$
    where
    $$f_{i, k}(\vec{a}) 
    \;=\; \begin{cases}
         c_{i, k}  & \text{if } a_{g(i', j)} = 1 ~ \forall 1 \leq i' \leq n, 1 \leq j \leq d_{i, k, i'}\\
         0 & \text{otherwise}\end{cases}$$
    Then, assuming we can guarantee $s_{g(i,j)} = s_i$, the expected payoff of strategy~$1$ to player~$i$ is exactly $\Prob(M_i^{O_{\vec x}}() = 1)$, while the payoff of strategy~$0$ is always $p_i$; hence, as desired, the Nash equilibrium conditions force~$i$ to choose~$1$ if the probability is greater than~$p_i$, $0$ if it is smaller.
    
    It remains to choose the payoffs $(u_{g(i,j)}(\vec a), u_{h(i,j)}(\vec a))$ of the copy and auxiliary players. In order to force $s_{g(i,j)} = s_i$, we set these payoffs as follows:
    \begin{center}
      \begin{tabular}{| c | c | c|}
        \hline 
        \multicolumn{3}{|c|}{$a_{i} = 0$}
        \\
        \hline
        & $a_{h(i, j)} = 0$ & $a_{h(i, j)} = 1$ \\
        \hline
        $a_{g(i, j)} = 0$ & $(1, 0)$ & $(0, 0)$ \\
        \hline
        $a_{g(i, j)} = 1$ & $(0, 1)$ & $(1, 0)$ \\
        \hline
      \end{tabular}
      \\
      \vspace{2em}
      \begin{tabular}{| c | c | c|}
        \hline 
        \multicolumn{3}{|c|}{$a_{i} = 1$}
        \\
        \hline
        & $a_{h(i, j)} = 0$ & $a_{h(i, j)} = 1$ \\
        \hline
        $a_{g(i, j)} = 0$ & $(1, 0)$ & $(0, 1)$ \\
        \hline
        $a_{g(i, j)} = 1$ & $(0, 0)$ & $(1, 0)$ \\
        \hline
      \end{tabular}
    \end{center}
    
    We show in Appendix~\ref{appendix:matching-pennies} that at Nash equilibrium, these payoffs force~$s_{g(i,j)} = s_i$.
\end{proof}

    Theorem~\ref{theorem:existence-from-nash-equilibria} is a special case of Theorem~\ref{theorem:existence}(i). The proof can be adapted to also show an analog of Theorem~\ref{theorem:existence}(ii), but we omit the details here.

\section{Related Work}

Joyce and Gibbard \cite{Joyce:1998} describe one justification for mixed Nash equilibria in terms of causal decision theory.  Specifically, they discuss a \emph{self-ratification} condition that extends CDT to cases when one's action is evidence of different underlying conditions that might change which actions are rational.  An action self-ratifies if and only if it causally maximizes expected utility in a world model that has been updated on the evidence that this action is taken. 

For example, consider the setting of a matching pennies game where players can predict each other accurately.  The fact that player A plays ``heads'' is evidence that player B will predict that player A will play ``heads'' and play ``tails'' in response, so player A would then have preferred to play ``tails'', and so the ``heads'' action would fail to self-ratify.  However, the mixed strategy of flipping the coin \emph{would} self-ratify.  Our reflection principle encodes some global constraints on players' mixed strategies that are similar to self-ratification.

The question of how to model agents as an ordinary part of the environment is of interest in the speculative study of human-level and smarter-than-human artificial intelligence \cite{Orseau:2012,Soares:2014}. Although such systems are still firmly in the domain of futurism, there has been a recent wave of interest in foundational research aimed at understanding their behavior, in order to ensure that they will behave as intended if and when they are developed~\cite{FLI:2015,Bostrom:2014,Soares:2014}.

Theoretical models of smarter-than-human intelligence such as Hutter's universally intelligent agent AIXI~\cite{Hutter:2005} typically treat the agent as separate from the environment, communicating only through well-defined input and output channels. In the real world, agents run on hardware that is part of the environment, and Orseau and Ring~\cite{Orseau:2012} have proposed formalisms for studying \emph{space-time embedded intelligence} running on hardware that is embedded in its environment. Our formalism might be useful for studying idealized models of agents embedded in their environment: While real agents must be boundedly rational, the ability to study perfectly Bayesian space-time embedded intelligence might help to clarify which aspects of realistic systems are due to bounded rationality, and which are due to the fact that real agents aren't cleanly separated from their environment.

\section{Conclusions and Future Work}

In this paper, we have introduced \emph{reflective oracles}, a type of probabilistic oracle which is able to answer questions about the behavior of oracle machines with access to the same oracle. We've shown that such oracle machines can implement a version of causal decision theory, and used this to establish a close relationship between reflective oracles and Nash equilibria.

We have focused on answering queries about oracle machines that halt with probability~1, but the reflection principle presented in Section~\ref{section:reflective-oracles} can be modified to apply to machines that do not necessarily halt. To do so, we replace the condition
$$\Prob(M^O() = 1) \;<\; p \;\;\Longrightarrow\;\; \Prob(O(M, p) = 0) \;=\; 1$$
by the condition
$$\Prob(M^O() \neq 0) \;<\; p \;\;\Longrightarrow\;\; \Prob(O(M, p) = 0) \;=\; 1$$
This is identical to the former principle if $M^O()$ is guaranteed to halt, but provides sensible information even if there is a chance that $M^O()$ loops. Appendix~\ref{appendix:existence-proof} proves the existence of reflective oracles satisfying this stronger reflection principle.

The ability to deal with non-halting machines opens up the possibility of applying reflective oracles to simplicity priors such as Solomonoff induction~\cite{Solomonoff:1964}, which defines a probability distribution over infinite bit sequences by, roughly, choosing a random program and running it. Solomonoff induction deals with computable hypotheses, but is itself uncomputable (albeit computably approximable) because it must deal with the possibility that a randomly chosen program may go into an infinite loop after writing only a finite number of bits on its output tape. A reflective oracle version of Solomonoff induction would be able to deal with a hypothesis space consisting of arbitrary oracle machines, while itself being implementable as an oracle machine; this would make it possible to model a predictor which predicts an environment it is itself embedded in. We leave details to future work.

\printbibliography

\section*{APPENDIX}

\appendix

\section{Nash Equilibria in a Variant of Matching Pennies} \label{appendix:matching-pennies}

\begin{lemma}
Consider an $n$-player game with three distinguished players, each of which has two pure strategies: Player Row has strategies Up and Down, player
Column has strategies Left and Right, and player Matrix has strategies Front and Back. Suppose that the payoffs of (Row, Column) depend only on the strategies of these three players, as follows:
\begin{center}
\begin{tabular}{|c | c|}
  \hline $(1, 0)$ & $(0, 0)$ \\ \hline $(0, 1)$ & $(1, 0)$ \\ \hline
\end{tabular}
\hspace{2em}
\begin{tabular}{|c | c|}
  \hline $(1, 0)$ & $(0, 1)$ \\ \hline $(0, 0)$ & $(1, 0)$ \\ \hline
\end{tabular}
\end{center}
where the first matrix indicates the payoffs when Matrix plays Front, and the second matrix indicates the payoffs when Matrix plays Back.

Write $p$ for the probability that Row plays Down, and
$q$ for the probability that Matrix plays Back. At Nash
equilibrium, we have $p = q$.
\end{lemma}

\begin{proof}
\begin{itemize}
  \item
Case 1: $0 < q < 1$.

Suppose that there is a Nash equilibrium where Column plays Left. Then
Row would play Up, but then Column would strictly prefer Right,
which is a contradiction.

Suppose that there is a Nash equilibrium where Column plays Right. Then
Row would play Down, but then Column would strictly prefer Left,
which is a contradiction.

Thus, at every Nash equilibrium, Column must mix between strategies.
Hence, at equilibrium, Column must be indifferent between Left and
Right.
This is equivalent to $p(1-q) = (1-p)q$.
This implies $p>0$, since otherwise we'd have $0(1-q) = (1-0)q$,
i.e. $0 = q$, but we assumed $0 < q < 1$.
Thus, we can divide the equation by $pq$, yielding:
\begin{align*}
  &(1-q)/q = (1-p)/p \\ \Leftrightarrow\; &1/x - 1 = 1/p - 1
  \\ \Leftrightarrow\; &1/q = 1/p \\ \Leftrightarrow\; &q = p
\end{align*}
\item
    Case 2: $q = 0$.

    This gives us the following payoff matrix:

    \begin{center}
\begin{tabular}{|c | c|}
  \hline $(1, 0)$ & $(0, 0)$ \\ \hline $(0, 1)$ & $(1, 0)$ \\ \hline
\end{tabular}
\end{center}

Suppose that there is a Nash equilibrium with $p>0$. Then at this
equilibrium, Column must play Left; but if Column plays Left, then Row
strictly prefers Up, which contradicts $p>0$. Hence, we must have $p =
0 = q$.

\item
Case 3: $q = 1$.

This gives us the following payoff matrix:

\begin{center}
\begin{tabular}{|c | c|}
  \hline $(1, 0)$ & $(0, 1)$ \\ \hline $(0, 0)$ & $(1, 0)$ \\ \hline
\end{tabular}
\end{center}

Suppose that there is a Nash equilibrium with $p<1$. Then at this
equilibrium, Column must play Right; but if Column plays Right, then
Row strictly prefers Down, which contradicts $p<1$. Hence, we must have
$p = 1 = q$.

\end{itemize}
\end{proof}

\section{Proof of the Existence Theorem} \label{appendix:existence-proof}

In this appendix, we prove Theorem~\ref{theorem:existence}(ii). Thus, suppose that~$R$ is a set of queries and~$O$ is some oracle; we want to show the existence of an oracle~$O'$ which is reflective on~$R$ and satisfies $\Prob(O'(M,p) = 1) = \Prob(O(M,p) = 1)$ for all $(M,p)\notin R$.

\newcommand{\query}{\mathrm{query}}
\newcommand{\eval}{\mathrm{eval}}
\newcommand{\cM}{\mathcal{M}}
\newcommand{\NN}{\mathbb{N}}
\newcommand{\QQ}{\mathbb{Q}}
\newcommand{\RR}{\mathbb{R}}
\newcommand{\Pow}{\mathrm{Pow}}

We will describe the behavior of~$O'$ by a pair of functions,
$\query : \mathcal{M}\times([0,1]\cap\mathbb{Q})\to[0,1]$ and $\eval : \cM\to[0,1]$. The first of these gives the distribution of $O'$, i.e., $\query(M,p) = \Prob(O'(M,p) = 1)$. The second gives the distribution of a machine's behavior under~$O'$: If~$M$ almost surely returns either $0$ or $1$, then $\eval(M) = \Prob(M^{O'}() = 1)$.

Function pairs $(\query,\eval)$ can be seen as elements of $A := [0,1]^{\cM\times([0,1]\cap\QQ)}\times[0,1]^\cM$, which is a convex and compact subset of the locally convex topological vector space $\RR^{\cM\times([0,1]\cap\QQ)}\times\RR^\cM$ (with the product topology). We now define a correspondence $f : A\to\Pow(A)$, such that fixed points $(\query,\eval)\in f(\query,\eval)$ yield oracles $O'$ of the desired form.

We define $f$ by giving a set of necessary and sufficient conditions for $(\query',\eval')\in f(\query,\eval)$. We place three conditions on $\query'(M,p)$: If $(M,p)\in R$ and $\eval(M) > p$, then $\query'(M,p) = 1$; if $(M,p)\in R$ and $\eval(M) < p$, then $\query'(M,p) = 0$; and if $(M,p)\notin R$, then $\query'(M,p) = \Prob(O(M,p) = 1)$.

To describe the conditions on $\eval'(M)$, we will consider the definition of ``probabilistic oracle machine'' to include the initial state of the machine's working tapes, so that we can view the state of a machine $M^O()$ after one step of computation as a new machine $N^O()$. Then, any machine $M$ can be classified as performing one of the following operations as its first step of computation: (i)~a deterministic computation step, yielding a new state $N$, in which case $\eval'(M) = \eval(N)$; (ii)~a coin flip, yielding a state~$N$ with a rational probability~$p$ and another state~$N'$ with probability~$1-p$, in which case $\eval'(M) = p\cdot\eval(N) + (1-p)\cdot\eval(N')$; (iii)~halting, with the output tape containing~$0$ (in which case $\eval'(M) = 0$) or~$1$ (in which case $\eval'(M) = 1$) or some other output (in which case $\eval'(M)$ is arbitrary); or (iv)~an invocation of the oracle on a pair $(M',p)$, yielding a new state~$N$ if the oracle returns~$0$ and a different new state~$N'$ if it returns~$1$. In the last case, writing $q := \query(M',p)$, the condition is $\eval'(M) = (1-q)\cdot\eval(N) + q\cdot\eval(N')$.

Given a fixed point $(\query,\eval)\in f(\query,\eval)$, define $O'$ by $\Prob(O'(M,p) = 1) = \query(M,p)$. Then, it can be shown by induction that for every $T\in\NN$ and every $M\in\cM$, $\eval(M)$ is $\ge$ the probability that $M^{O'}()$ returns $1$ after at most $T$ timesteps, and $\le$ the probability that it returns something other than $0$ within this time bound; in the limit, we obtain
$$\Prob(M^{O'}() = 1) \;\le\; \eval(M) \;\le\; \Prob(M^{O'}() \neq 0) 
$$
Together with the conditions on $\query(M,p)$, this shows that
$$\begin{aligned}
\Prob(M^{O'}() = 1) > p \implies \Prob(O'(M,p) = 1) = 1 \\
\Prob(M^{O'}() = 0) > (1-p) \implies \Prob(O'(M,p) = 0) = 1
\end{aligned}$$
which is a strengthening of the conditions of Section~\ref{section:reflective-oracles}: it is equivalent in the case where $M^{O'}()$ halts with probability~1, but provides information even if $M^{O'}()$ may fail to halt.

It remains to be shown that $f(\cdot)$ has a fixed point. To do so, we employ the infinite-dimensional generalization of Kakutani's fixed-point theorem~\cite{Fan:1952}.

It is clear from the definition that $f(\query,\eval)$ is non-empty, closed and convex for all $(\query,\eval)\in A$. Hence, to show that~$f$ has a fixed point, it is sufficient to show that it has closed graph.

Thus, assume that we have sequences $(\query_n,\eval_n)\to (\query,\eval)$ and $(\query'_n,\eval'_n)\to(\query',\eval')$, such that $(\query'_n,\eval'_n)\in f(\query_n,\eval_n)$ for every~$n$; we need to show that then, $(\query',\eval')\in f(\query,\eval)$.

For the conditions on $\eval'$, we can simply take the limit $n\to\infty$ on both sides of each equation. The condition on $\query'(M,p)$ for $(M,p)\notin R$ is clearly fulfilled, since $\query'_n(M,p)$ is constant in this case. The two remaining conditions on $\query'(M,p)$ are entirely symmetrical; without loss of generality, consider the case $\eval(M) > p$, $(M,p)\in R$.

In this case, since $(\query_n,\eval_n)\to(\query,\eval)$ and convergence is pointwise, there must be an $n_0$ such that $\eval_n(M) > p$ for all $n\ge n_0$. Since $(\query'_n,\eval'_n)\in f(\query_n,\eval_n)$, it follows that $\query'_n(M,p) = 1$ for all $n\ge n_0$, whence $\query'(M,p) = 1$ as desired. This completes the proof.

\end{document}